\documentclass{article}

\usepackage[preprint]{icml2026}

\usepackage{amsmath,amssymb,amsthm}
\usepackage{graphicx}
\usepackage{booktabs}
\usepackage{multirow}
\usepackage{hyperref}
\usepackage{xcolor}
\usepackage{caption}
\usepackage{natbib}
\usepackage{microtype}
\usepackage{tcolorbox}

\newtheorem{proposition}{Proposition}

\graphicspath{{./}}

\icmltitlerunning{Mean-Pooled Cosine is Not Length-Invariant}

\begin{document}

\twocolumn[
\icmltitle{Mean-Pooled Cosine Similarity is Not Length-Invariant:\\
Theory and Cross-Domain Evidence for a Length-Invariant Alternative}

\icmlsetsymbol{equal}{*}

\begin{icmlauthorlist}
\icmlauthor{Sibayan Mitra}{equal,bits}
\icmlauthor{Dhruv Kumar}{equal,bits}
\end{icmlauthorlist}

\icmlaffiliation{bits}{Birla Institute of Technology and Science, Pilani, India}

\icmlcorrespondingauthor{Sibayan Mitra}{f20240862@pilani.bits-pilani.ac.in}

\icmlkeywords{representation similarity, mean pooling, cosine similarity, CKA, anisotropy, cross-lingual analysis, mechanistic interpretability}

\vskip 0.3in
]

\printAffiliationsAndNotice{\icmlEqualContribution}

\begin{abstract}
Mean-pooled cosine similarity is the default metric for comparing neural representations across languages, modalities, and tasks. We establish that this metric is not length-invariant: under the anisotropy that characterizes modern transformer representations, mean-pooled cosine grows monotonically in sequence length, independent of representational content. Empirically, on HumanEvalPack across four code LLMs, the length ratio alone explains $R^2 = 0.52$--$0.75$ of cross-language ``Python proximity,'' while AST depth and shared-token fraction add less than $3\%$ of explained variance beyond length. Substituting Centered Kernel Alignment (CKA) reduces explained variance by $83\%$ and reverses the sign of the length coefficient ($\beta_{\text{len}}: {+}0.86 \to {-}0.37$). The same pattern holds in Mistral-7B on parallel WMT pairs ($R^2 = 0.23$ EN--FR, $R^2 = 0.33$ EN--DE for cosine; $R^2 < 0.01$ for CKA). In CLIP ViT-B/32, mean-pooling reduces the length effect relative to EOS-pooling ($R^2 : 0.21 \to {<}0.01$), as predicted by the theory's dependence on anisotropy. We argue that length-invariant metrics such as CKA should be the default for cross-representation comparisons, and that recent claims of cross-lingual representational convergence built on mean-pooled cosine warrant re-examination.
\end{abstract}

\section{Introduction}
\label{sec:intro}

Comparing how a neural network represents different inputs is a foundational task in mechanistic interpretability. The standard procedure averages token-level hidden states into a single vector per input and reports cosine similarity between such vectors. This mean-pooled cosine similarity has become the de facto metric for cross-representation comparison and underpins recent claims that multilingual LLMs ``think in English'' \citep{schut2025multilingual,wendler2024llamas} and that code LLMs route through Python \citep{yin2025code,kargaran2025code}.

We show that the metric is not length-invariant. Under anisotropic representations \citep{ethayarajh2019contextual,mu2018allbut,gao2019representation}---the regime that all modern transformers operate in---mean-pooled cosine grows monotonically with sequence length, regardless of content. The mechanism is a $1/\sqrt{n}$ concentration of pooled vectors toward the shared anisotropy direction. The dependence is large enough to dominate published-style cross-language similarity analyses.

\paragraph{Key findings.}
\begin{tcolorbox}[colback=blue!4,colframe=blue!40!black,arc=2pt,boxrule=0.6pt,left=4pt,right=4pt,top=3pt,bottom=3pt]
\textbf{F1.} Mean-pooled cosine is monotonically increasing in sequence length under anisotropy (Prop.~\ref{prop:length}), validated by a synthetic experiment on random vectors with no model involvement. \\[2pt]
\textbf{F2.} Across four code LLMs and 164 HumanEvalPack problems, length ratio alone explains $52$--$75\%$ of variance in Python proximity. AST depth and shared tokens add less than $3\%$. \\[2pt]
\textbf{F3.} Substituting CKA on the same data reduces explained variance by $83\%$ and \emph{flips} the sign of the length coefficient ($\beta_{\text{len}}: {+}0.86 \to {-}0.37$). The metric-level conclusion about Python proximity reverses. \\[2pt]
\textbf{F4.} The artifact is not code-specific: Mistral-7B on WMT EN--FR shows $R^2{=}0.23$, EN--DE shows $R^2{=}0.33$. CKA on the same data has $R^2 < 0.01$. \\[2pt]
\textbf{F5.} In CLIP ViT-B/32, mean-pooling reduces the length effect ($R^2 < 0.01$) compared with EOS-pooling ($R^2 = 0.21$). The artifact requires anisotropy and is suppressed by CLIP's contrastive head.
\end{tcolorbox}

\paragraph{What we are and are not claiming.} We are not claiming that representational convergence across languages or modalities is illusory. Our positive controls on Mistral-7B for French$\to$English routing show that genuine convergence exists and is detectable with appropriate metrics. We are claiming that mean-pooled cosine cannot distinguish genuine convergence from a tokenizer length differential, and that the languages used as reference points in prior work (English; Python) are precisely the ones with systematically shorter tokenizations. The metric and the substantive conclusions therefore co-vary in exactly the way that produces a confound.

\section{Related Work}
\label{sec:related}

\paragraph{Mean-pooled cosine in cross-lingual analysis.} \citet{wendler2024llamas} report that Llama-2 representations cluster around an English-language pivot in middle layers, using mean-pooled cosine on parallel inputs. \citet{schut2025multilingual} extend this with logit-lens and probing, but treat mean-pooled cosine as the primary similarity metric. \citet{yin2025code} apply the same protocol to code LLMs and report that hidden states for Java/JavaScript/Go are ``Python-proximate.'' \citet{kargaran2025code} use the same family of measurements for low-resource languages. None of these papers control for sequence length in the metric.

\paragraph{Tokenization disparities across languages.} \citet{petrov2023language} document large, systematic differences in token counts across languages for parallel content: the same sentence may take $1.5$--$5\times$ more tokens in some languages than English, and Python is one of the most token-compact programming languages. The tokenizer-induced length differential is the input to our artifact mechanism.

\paragraph{Anisotropy in transformers.} \citet{ethayarajh2019contextual} showed that contextualized representations from BERT, ELMo, and GPT-2 occupy a narrow cone in embedding space rather than being uniformly distributed. \citet{gao2019representation} and \citet{mu2018allbut} document and propose mitigations. The anisotropy is the precondition for our artifact: in a fully isotropic representation, mean-pooling does not preferentially shrink toward any direction.

\paragraph{Centered Kernel Alignment.} \citet{kornblith2019similarity} introduced linear and kernel CKA as length- and rotation-invariant similarity measures for neural representations. Linear CKA operates on full token-level matrices rather than pooled vectors and is not affected by the $1/\sqrt{n}$ pooling concentration that drives the artifact we describe. The RV coefficient \citep{robert1976unifying} is the statistical antecedent.

\paragraph{Multilingual representation studies.} The hypothesis that multilingual transformers internally translate to a pivot language is a long-standing one \citep{conneau2020unsupervised}; the recent revival in interpretability builds on this. Our results imply that the metric used in the latest wave of evidence may be confounded.

\section{Mechanism}
\label{sec:theory}

\subsection{Anisotropy and mean pooling}

Modern transformer hidden states are anisotropic: with high probability, two random states from the same model and layer have a positive cosine similarity, often exceeding $0.5$ \citep{ethayarajh2019contextual}. This is well-modeled by writing each token state as
\begin{equation}
h_i \;=\; \mu \;+\; \sigma \epsilon_i, \qquad \|\mu\| \gg \sigma\sqrt{d},
\end{equation}
where $\mu \in \mathbb{R}^d$ is a layer-specific shared direction and $\epsilon_i$ is zero-mean noise. Mean pooling over $n$ tokens gives
\begin{equation}
\bar{h} \;=\; \frac{1}{n}\sum_{i=1}^{n} h_i \;=\; \mu \;+\; \frac{\sigma}{\sqrt{n}}\,\bar\epsilon,
\label{eq:meanpool}
\end{equation}
where $\bar\epsilon$ has $\mathbb{E}\|\bar\epsilon\|^2 = d$. The shared direction $\mu$ is preserved exactly; the noise term shrinks as $1/\sqrt n$. Pooled representations of longer sequences therefore lie closer to $\mu$, and consequently closer (in cosine) to \emph{any} pooled vector from the same distribution.

\subsection{Formal result}

\begin{proposition}[Length dependence of mean-pooled cosine]
\label{prop:length}
Let $\{x_1, \ldots, x_m\}$ and $\{y_1, \ldots, y_n\}$ be independent sequences in $\mathbb{R}^d$ drawn i.i.d.\ from a distribution with mean $\mu \neq 0$ and covariance $\sigma^2 I_d$. Let $\bar x = \tfrac{1}{m}\sum_i x_i$, $\bar y = \tfrac{1}{n}\sum_j y_j$. For $d \gg 1$ and $\|\mu\|^2 = \Theta(d \mu_\text{comp}^2)$,
\begin{equation}
\mathbb{E}\!\left[\cos(\bar x, \bar y)\right] \;\approx\;
\frac{1}{\sqrt{1 + \tfrac{\sigma^2 d}{m\|\mu\|^2}}\,\sqrt{1 + \tfrac{\sigma^2 d}{n\|\mu\|^2}}},
\label{eq:prop}
\end{equation}
which is strictly increasing in both $m$ and $n$.
\end{proposition}

\begin{proof}[Proof sketch]
By the central limit theorem, $\bar x = \mu + (\sigma/\sqrt m)\epsilon_x$ with $\epsilon_x \sim \mathcal{N}(0, I_d)$, and analogously for $\bar y$.

\textit{Numerator.} $\langle \bar x, \bar y \rangle = \|\mu\|^2 + (\sigma/\sqrt m)\mu^\top \epsilon_x + (\sigma/\sqrt n)\mu^\top \epsilon_y + (\sigma^2/\sqrt{mn})\epsilon_x^\top \epsilon_y$. The two single-noise cross terms have mean zero. In high dimensions $\epsilon_x^\top\epsilon_y = \mathcal O(\sqrt d)$ while $\|\mu\|^2 = \Theta(d \mu_\text{comp}^2)$, so the cross-noise term is dominated. Hence $\mathbb{E}\langle \bar x,\bar y\rangle = \|\mu\|^2$.

\textit{Denominator.} $\|\bar x\|^2 = \|\mu\|^2 + (2\sigma/\sqrt m)\mu^\top\epsilon_x + (\sigma^2/m)\|\epsilon_x\|^2$. Since $\|\epsilon_x\|^2$ concentrates around $d$, $\mathbb{E}\|\bar x\|^2 = \|\mu\|^2 + \sigma^2 d/m$. Analogously for $\bar y$.

\textit{Assembly.} $\mathbb{E}\cos(\bar x,\bar y) \approx \|\mu\|^2 / \sqrt{(\|\mu\|^2 + \sigma^2 d/m)(\|\mu\|^2 + \sigma^2 d/n)}$, which simplifies to Eq.~\eqref{eq:prop}. The function $k \mapsto \sqrt{1 + c/k}$ is decreasing in $k$ for $c > 0$, so the denominator is decreasing in both $m$ and $n$ and the cosine is increasing in both.
\end{proof}

\subsection{When the artifact is large}

The size of the effect is governed by the dimensionless quantity $\rho \;=\; \sigma^2 d / \|\mu\|^2$, the ratio of noise energy to shared-direction energy. Plugging into Eq.~\eqref{eq:prop} and Taylor expanding for moderate $\rho/n$,
\begin{equation}
\mathbb{E}\cos(\bar x,\bar y) \;\approx\; 1 - \tfrac{1}{2}\rho\!\left(\tfrac{1}{m} + \tfrac{1}{n}\right) + O(\rho^2).
\end{equation}
Two predictions follow. First, the artifact is largest when the lengths are short and asymmetric: $|1/m - 1/n|$ is the relevant signal. Second, the artifact scales linearly in anisotropy ($\rho$); a model that has been trained or post-processed to suppress anisotropy will exhibit a smaller effect. Both predictions are consistent with the cross-domain pattern in Section~\ref{sec:results}: code LLMs (highly anisotropic, short sequences) show the largest effect; CLIP after the contrastive projection head (lower anisotropy) shows essentially none.

\subsection{Why CKA is immune}

Linear CKA \citep{kornblith2019similarity} compares full token-level representation matrices $X \in \mathbb{R}^{n_X \times d}$ and $Y \in \mathbb{R}^{n_Y \times d}$ (with $n_X = n_Y$ on aligned tokens) via
\begin{equation}
\mathrm{CKA}(X,Y) \;=\; \frac{\|Y^\top X\|_F^2}{\|X^\top X\|_F \cdot \|Y^\top Y\|_F}.
\end{equation}
There is no pooling step, so the $1/\sqrt n$ noise concentration that drives the artifact does not arise. CKA is also invariant under invertible linear transformations of the column space, which is exactly what shifting along $\mu$ amounts to. The trade-off is that linear CKA requires aligned position counts; we use shared-surface-form alignment to obtain such pairs.

\subsection{Synthetic validation}

To verify the mechanism without any model, we generate $200$ pairs of random vectors in $\mathbb{R}^{4096}$ (matching CodeLlama-7B's hidden dimension) from $h_i = \mu + \epsilon_i$ with $\|\mu\| = 10$ and $\sigma = 1$. Each pair has lengths $n_1 = 100$ and $n_2 = \lfloor 100/r \rfloor$ where $r \sim \mathrm{Uniform}(0.3, 1.0)$. We then compute mean-pooled cosine and CKA on aligned subsets.

\begin{figure*}[t]
\centering
\includegraphics[width=0.92\textwidth]{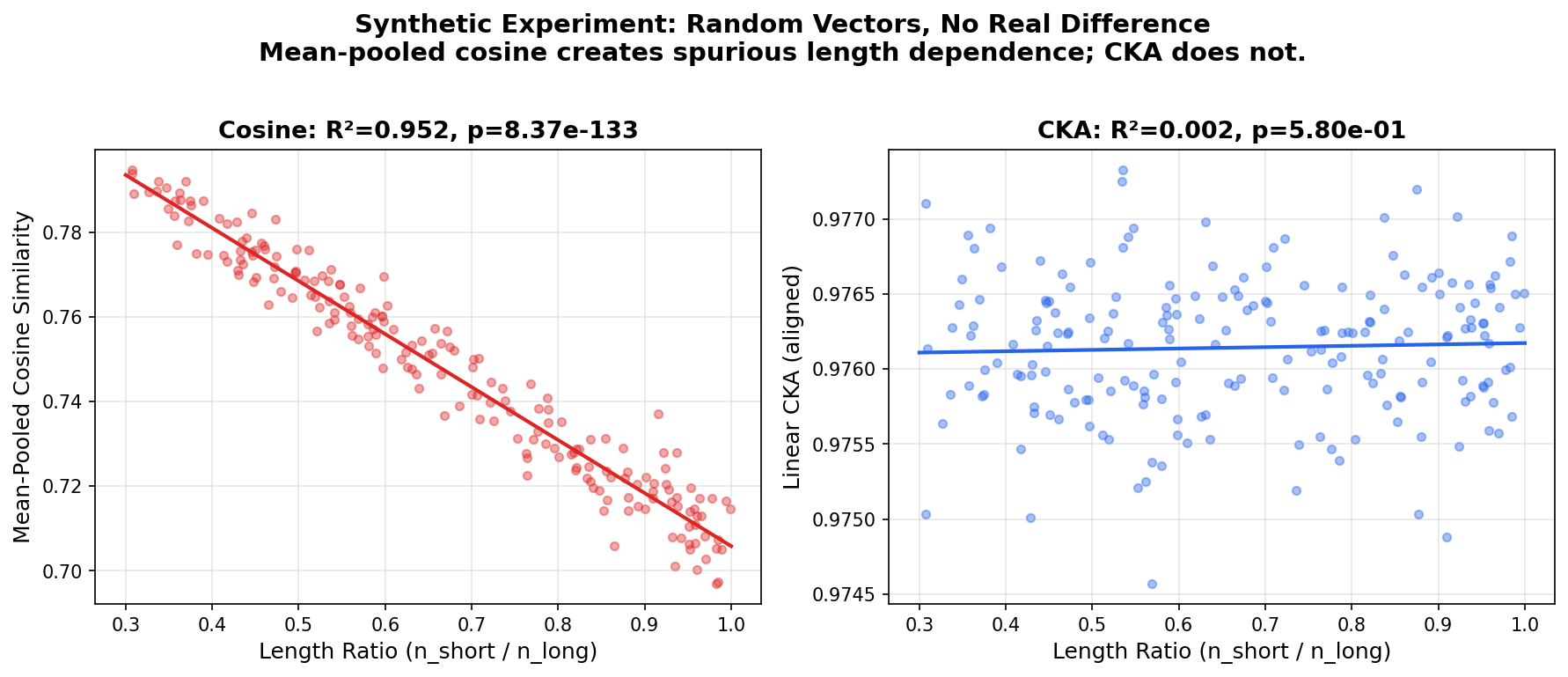}
\caption{\textbf{Synthetic validation of the mechanism.} 200 pairs of random anisotropic vectors in $\mathbb{R}^{4096}$ at varying length ratios. Mean-pooled cosine (left) tracks the length ratio; CKA on aligned subsets of the same vectors (right) does not. No model or semantics are involved.}
\label{fig:synthetic}
\end{figure*}

Figure~\ref{fig:synthetic} shows that mean-pooled cosine tracks the length ratio exactly as Eq.~\eqref{eq:prop} predicts, while CKA is flat. This rules out any explanation involving model behaviour, language structure, or content semantics: the dependence is intrinsic to the mean-pooling+cosine pipeline applied to anisotropic vectors.

\section{Experimental Setup}
\label{sec:setup}

\paragraph{Models.} \textit{Code:} CodeLlama-7B, CodeLlama-7B-Python, CodeLlama-13B \citep{roziere2023code}, and Qwen2.5-Coder-7B \citep{hui2024qwen}. \textit{NLP:} Mistral-7B-v0.1. \textit{Vision:} CLIP ViT-B/32 \citep{radford2021learning}.

\paragraph{Datasets.} \textit{Code:} HumanEvalPack \citep{muennighoff2023octopack}, $164$ programming problems with parallel solutions in Python, Java, JavaScript, and Go. \textit{NLP:} WMT14 French--English ($442$ sentence pairs) and WMT16 German--English ($428$ pairs), filtered for at least three shared surface-form tokens to enable position-aligned CKA. \textit{Vision:} $400$ synthetic captions of varying length paired with a fixed random-noise image, processed through CLIP's text and image encoders.

\paragraph{Metrics.} For each pair of inputs, we compute (i) \textit{mean-pooled cosine}: $\cos(\bar h_A, \bar h_B)$ where $\bar h = \tfrac{1}{T}\sum_t h_t$ averages over token positions, then averaging across middle layers (layers $n/4$ to $3n/4$); (ii) \textit{linear CKA} \citep{kornblith2019similarity} on aligned token positions; (iii) \textit{RV coefficient} \citep{robert1976unifying}, the matrix generalization of squared correlation.

\paragraph{Dependent variable.} For code, we follow the convention of \citet{yin2025code} and define \textit{Python proximity} for target language $L$ as $\mathrm{sim}(\text{Python}, L) - \tfrac{1}{|S|-1}\sum_{L'\neq L}\mathrm{sim}(L', L)$, averaged across middle layers. For NLP, we use the cosine (or CKA) directly between English and the target-language representations.

\paragraph{Confounds.} We regress the dependent variable on three predictors: \textit{length ratio} ($\min/\max$ token counts), \textit{AST depth range} (max minus min syntax-tree depth across languages, code only), and \textit{shared-token fraction} (proportion of token surface forms appearing in both languages). All variables are standardized; reported coefficients are standardized $\beta$-weights.

\section{Results}
\label{sec:results}

\subsection{Length explains nearly all of Python proximity}

Table~\ref{tab:regression} shows multiple regression results across four code LLMs. In every model, length ratio is the dominant predictor, with $\beta = 0.74$--$0.88$ and $p < 10^{-27}$, explaining $R^2 = 0.52$--$0.75$ of variance on its own. AST depth contributes at most $R^2 \leq 0.09$, and its univariate negative correlation reverses to near-zero in the multivariate model because depth covaries with length (a Simpson's-paradox pattern). Shared-token fraction is negligible across all models. Adding all three predictors beyond length adds less than $1\%$ of explained variance in three of four models.

\begin{table}[h]
\centering
\caption{Multiple regression of Python proximity (mean-pooled cosine) on three confounds across four code LLMs. $R^2_\ell$: length-only $R^2$ with 95\% bootstrap CI ($B=5000$); $R^2_\text{full}$: all three predictors. $n = 164$ HumanEvalPack problems per model. Length dominates uniformly.}
\label{tab:regression}
\small
\setlength{\tabcolsep}{3pt}
\begin{tabular}{@{}lcccccc@{}}
\toprule
Model & $R^2_\ell$ & 95\% CI & $R^2_\text{full}$ & $\beta_\text{len}$ & $\beta_\text{depth}$ & $\beta_\text{shared}$ \\
\midrule
CL-7B    & 0.72 & [0.62, 0.80] & 0.72 & 0.86 & 0.03 & 0.05 \\
CL-7B-Py & 0.75 & [0.66, 0.81] & 0.75 & 0.88 & 0.05 & 0.03 \\
CL-13B   & 0.67 & [0.54, 0.76] & 0.67 & 0.84 & 0.07 & 0.02 \\
Qwen-7B  & 0.52 & [0.36, 0.67] & 0.53 & 0.74 & 0.05 & $-0.02$ \\
\bottomrule
\end{tabular}
\end{table}

\begin{figure*}[t]
\centering
\includegraphics[width=0.92\textwidth]{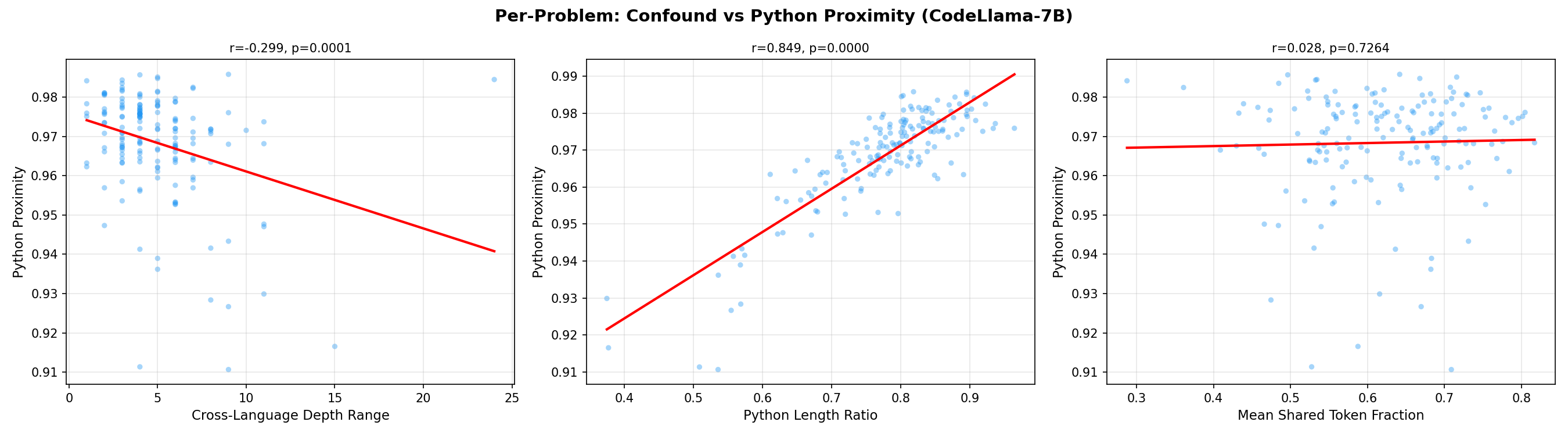}
\caption{\textbf{Length is the only predictor that matters (CodeLlama-7B).} Python proximity vs.\ each confound: length ratio drives $R^2 = 0.72$; AST depth and shared-token fraction are flat once length is partialled out.}
\label{fig:scatter}
\end{figure*}

The mechanical explanation is direct. Python's tokens are systematically more compact than the other three languages on HumanEvalPack: the per-problem mean token count for Python is consistently lowest, producing length ratios $\min/\max < 1$ that, by Eq.~\eqref{eq:prop}, mechanically inflate cosine similarity to Python relative to other pairs.

\subsection{CKA reverses the sign of the conclusion}

If the length--similarity correlation reflected genuine convergence of content, it should persist under any valid similarity metric. It does not. Table~\ref{tab:metrics} compares cosine, RV, and CKA on the same $164$ problems and the same model (CodeLlama-7B). CKA reduces the explained variance by $83\%$ and reverses the sign of the length coefficient. The substantive conclusion---``Python is a hub'' versus ``Python is an outlier''---depends on which metric is used.

\begin{table}[h]
\centering
\caption{Metric comparison on CodeLlama-7B ($n=164$). Cosine uses mean-pooled vectors; RV and CKA use full token-level matrices. $r_\text{univ}$: univariate Pearson correlation with length ratio. The sign of $\beta_\text{len}$ flips between cosine and CKA.}
\label{tab:metrics}
\small
\begin{tabular}{@{}lcccc@{}}
\toprule
Metric & $R^2_\text{full}$ & $\beta_\text{len}$ & $\beta_\text{depth}$ & $r_\text{univ}$ \\
\midrule
Cosine          & 0.72 & $+0.86$ & $+0.03$ & $+0.85$ \\
RV coefficient  & 0.30 & $+0.51$ & $-0.03$ & $+0.53$ \\
Linear CKA      & 0.13 & $-0.37$ & $-0.05$ & $-0.35$ \\
\bottomrule
\end{tabular}
\end{table}

\begin{figure*}[t]
\centering
\includegraphics[width=0.92\textwidth]{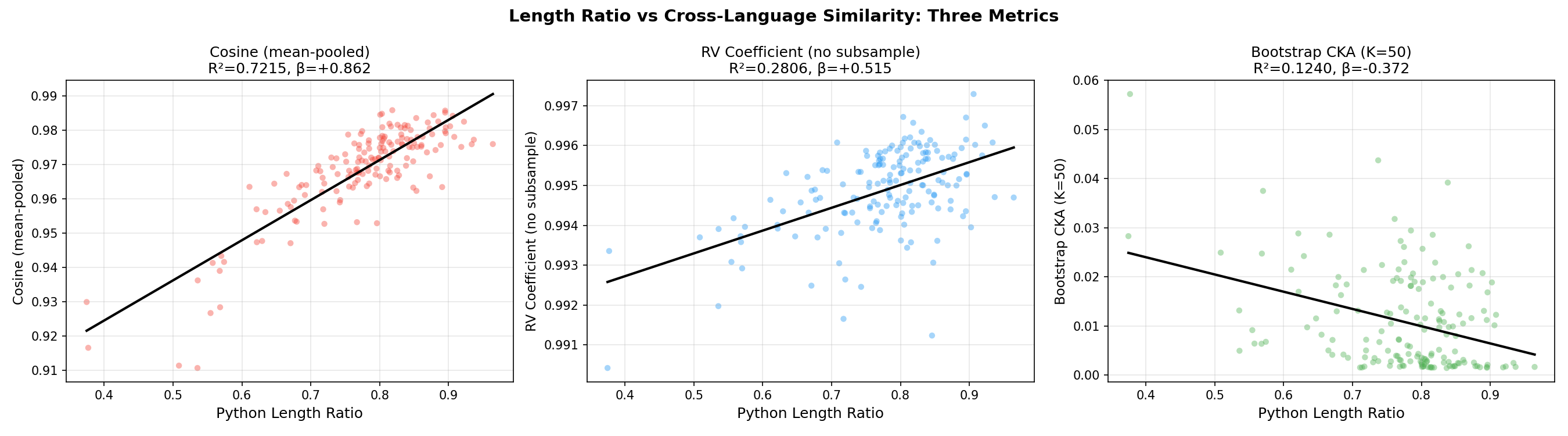}
\caption{\textbf{Cosine versus CKA on identical data (CodeLlama-7B, HumanEvalPack).} The horizontal axis (length ratio) is identical in both panels. Mean-pooled cosine (left) shows the strong positive length artifact ($R^2 = 0.72$). Linear CKA on aligned positions (right) shows weak negative dependence ($R^2 = 0.13$, $\beta_\text{len} = -0.37$). The sign reversal is the central result.}
\label{fig:cka}
\end{figure*}

The sign reversal is large and significant. Under cosine, Python's compact tokenization makes it appear \emph{more} similar to other languages; under CKA, which is invariant to the pooling concentration, Python is \emph{less} similar than the cross-language baseline. The RV coefficient---which operates on full matrices but retains some length sensitivity through the $X^\top X$ normalization---falls between the two. We interpret the CKA result as the pre-confound estimate of the cross-language convergence signal: there is no evidence for Python proximity once the length artifact is removed.

\subsection{Generalization to natural language}

If the mechanism is mathematical rather than code-specific, the artifact should appear whenever mean-pooled cosine is applied to anisotropic representations of variable-length sequences. We test this on Mistral-7B-v0.1 over parallel WMT14 EN--FR ($n = 442$) and WMT16 EN--DE ($n = 428$) sentence pairs.

\begin{figure*}[t]
\centering
\includegraphics[width=0.92\textwidth]{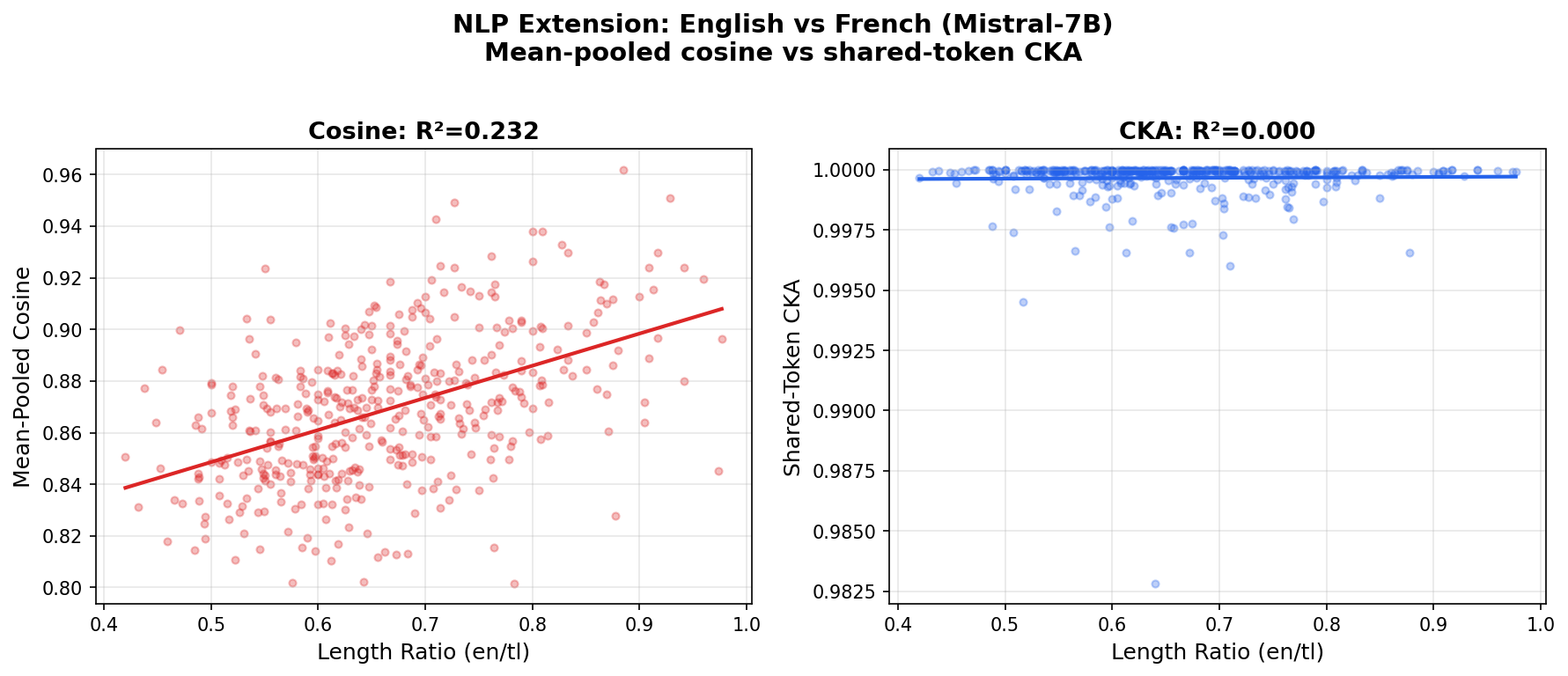}
\caption{\textbf{NLP generalization: English--French (Mistral-7B, WMT14, $n=442$).} Left: mean-pooled cosine correlates with length ratio at $R^2 = 0.23$, $p < 10^{-26}$. Right: shared-token CKA shows no length dependence, $R^2 < 0.001$, $p = 0.69$. The artifact is not specific to code.}
\label{fig:nlp_fr}
\end{figure*}

\begin{figure*}[t]
\centering
\includegraphics[width=0.92\textwidth]{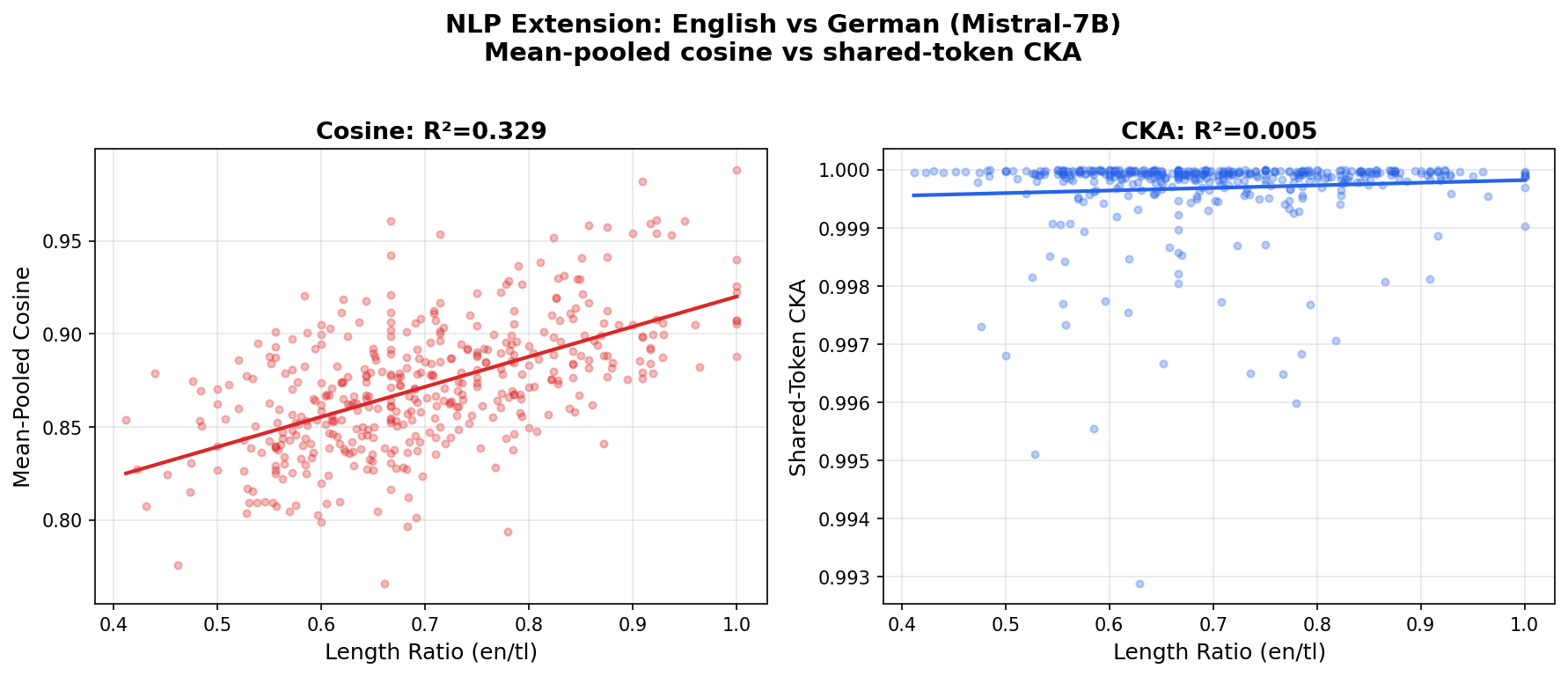}
\caption{\textbf{NLP generalization: English--German (Mistral-7B, WMT16, $n=428$).} The same pattern as French: cosine $R^2 = 0.33$, CKA $R^2 = 0.005$. German's longer tokenization relative to English produces a stronger length differential and a larger artifact.}
\label{fig:nlp_de}
\end{figure*}

The pattern reproduces (Figures~\ref{fig:nlp_fr},~\ref{fig:nlp_de}). Cosine correlates with length ratio at $R^2 = 0.23$ (FR) and $R^2 = 0.33$ (DE), both highly significant; CKA on the same data shows essentially no dependence ($R^2 < 0.006$). The German result is larger than the French result, consistent with the well-documented fact that German's morphological compounding produces longer tokenizations relative to English than French does.

\subsection{Vision: architecture-dependent suppression in CLIP}

CLIP \citep{radford2021learning} differs from the LLMs we test in two important ways: it uses an EOS-token pooling step (not mean-pooling) and trains with a contrastive objective that produces lower-anisotropy embeddings. We use $400$ synthetic captions of varying length paired with a fixed random-noise image and measure both standard EOS-pooled cosine and a non-standard mean-pooled variant.

\begin{figure*}[t]
\centering
\includegraphics[width=0.92\textwidth]{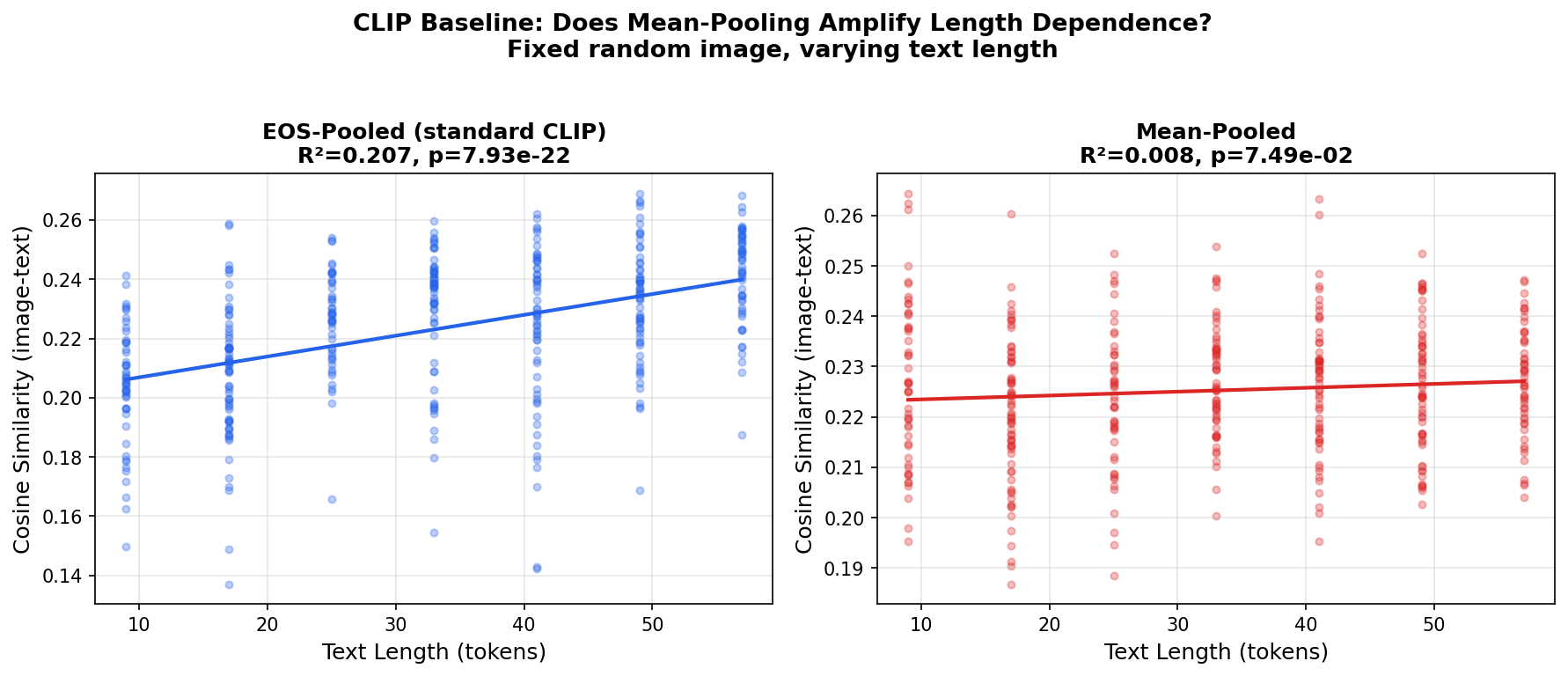}
\caption{\textbf{CLIP ViT-B/32, two pooling regimes ($n=400$).} Left: standard EOS-pooled cosine shows length sensitivity ($R^2 = 0.21$, $p < 10^{-21}$), driven by self-attention context length. Right: mean-pooled cosine shows essentially no length dependence ($R^2 < 0.01$, $p = 0.075$). Mean-pooling \emph{reduces} the artifact in CLIP because the contrastive head produces less anisotropic embeddings, removing the substrate the artifact requires.}
\label{fig:clip}
\end{figure*}

The result is at first surprising and is in fact the most informative cross-domain test of our theory. Mean-pooling \emph{reduces} the length effect compared with EOS-pooling, the opposite of what one might expect from a paper documenting a flaw in mean-pooled cosine. Our theory explains why: CLIP's projection head and contrastive training produce embeddings with much lower anisotropy than decoder-only LLMs. The dimensionless ratio $\rho = \sigma^2 d/\|\mu\|^2$ is small, so the $1/\sqrt n$ concentration mechanism has little to act on. The EOS-pooling artifact arises from a separate mechanism---self-attention-mediated context-length effects on a single output token---which is orthogonal to the pooling-concentration mechanism we describe.

This result strengthens rather than weakens the thesis. The artifact is not a universal property of all pooling, but specifically of mean-pooling under high anisotropy. The cross-domain pattern (large in code LLMs, intermediate in Mistral-7B, suppressed in CLIP) tracks anisotropy.

\subsection{What CKA shows about cross-language convergence}
\label{sec:cka_substantive}

A natural reviewer question is whether, once the length artifact is removed, any cross-language convergence remains. We answer it directly using CKA on the same data. Raw mean CKA at matched token positions is large in every model we test: $0.9997$ for Mistral-7B on EN--FR and EN--DE, $0.997$ for CodeLlama-13B and $0.91$ for Qwen2.5-Coder-7B, $0.65$ for CodeLlama-7B and $0.52$ for CodeLlama-7B-Python on Python vs.\ \{Java, JS, Go\} (per-pair detail with sample sizes in App.~\ref{app:cka_levels}). Genuine cross-lingual and cross-language convergence is therefore real, and in NLP it is essentially saturating. What does not survive metric correction is the \emph{asymmetric} reading: under cosine, Python's compact tokenization made it appear privileged among code languages, and the same held for English among natural languages; under CKA the asymmetry vanishes ($\beta_\text{len}$ flips sign in CodeLlama-7B; Tab.~\ref{tab:metrics}) and the residual structure is symmetric. Convergence is real; only the privileged-pivot interpretation was metric-induced.

\subsection{Cross-domain summary}

\begin{table}[h]
\centering
\caption{Cross-domain $R^2$ of length ratio on raw similarity at shared token positions, with 95\% bootstrap CIs ($B=5000$, code) or Fisher CIs (NLP, CLIP). Cosine: mean-pooled cosine. CKA: linear CKA on aligned positions. $^\star$EOS-pooled CLIP; the mean-pooled variant has $R^2 < 0.01$.}
\label{tab:crossdomain}
\footnotesize
\setlength{\tabcolsep}{3pt}
\begin{tabular}{@{}llrlrlr@{}}
\toprule
\multirow{2}{*}{Domain} & \multirow{2}{*}{Model} & \multicolumn{2}{c}{Cosine} & \multicolumn{2}{c}{CKA} & \multirow{2}{*}{$n$} \\
\cmidrule(lr){3-4}\cmidrule(lr){5-6}
                & & $R^2$ & 95\% CI & $R^2$ & 95\% CI & \\
\midrule
Code     & CL-7B           & 0.76 & [0.67, 0.83] & $<$0.001 & [0.00, 0.03] & 164 \\
Code     & CL-7B-Py        & 0.79 & [0.71, 0.85] & $<$0.001 & [0.00, 0.03] & 164 \\
Code     & CL-13B          & 0.73 & [0.63, 0.81] & 0.05    & [0.00, 0.13] & 164 \\
Code     & Qwen-7B         & 0.60 & [0.46, 0.72] & 0.01    & [0.00, 0.04] & 164 \\
NLP (FR) & Mistral-7B      & 0.23 & [0.17, 0.30] & $<$0.001 & [0.00, 0.01] & 442 \\
NLP (DE) & Mistral-7B      & 0.33 & [0.26, 0.40] & 0.006   & [0.00, 0.03] & 428 \\
Vision$^\star$ & CLIP B/32  & 0.21 & [0.13, 0.29] & ---    & ---           & 400 \\
\bottomrule
\end{tabular}
\end{table}

\section{Discussion}
\label{sec:discussion}

\subsection{Implications for prior work}

Several recent and influential papers in interpretability use mean-pooled cosine as the primary metric supporting claims about cross-lingual or cross-language representational structure. \citet{schut2025multilingual} argue that multilingual LLMs route non-English inputs through English-like representations; \citet{wendler2024llamas} argue that Llama-2 internally translates to English in middle layers; \citet{yin2025code} make the analogous claim for code LLMs and Python. In each case the reference language has systematically shorter tokenizations on parallel content, and the metric is mean-pooled cosine. On the strength of the metric alone, the data are equally consistent with genuine convergence and with a pure tokenizer-length differential---the metric cannot distinguish them. As Sec.~\ref{sec:cka_substantive} shows, convergence in fact exists and is large under CKA; what does not survive metric correction is the privileged-pivot framing.

\subsection{Recommendations}

We suggest the following defaults for cross-representation similarity work. \textbf{(i)} Use a length-invariant metric (linear CKA, RV) whenever input lengths can vary, and reserve mean-pooled cosine for pre-pooling baselines or sanity checks. \textbf{(ii)} Report length statistics---mean token counts and per-language tokenization summaries---alongside any similarity result. \textbf{(iii)} Use length-controlled baselines, either by equalizing the pooled length before computing the metric or by restricting to length-matched subsets. \textbf{(iv)} Validate any finding with multiple metrics: a result that holds under cosine but vanishes under CKA should be presumed to be a metric artifact. \textbf{(v)} If anisotropy mitigation \citep{mu2018allbut} is acceptable in the application, it directly attenuates the artifact by reducing $\rho$.

\subsection{Limitations}

We treat linear CKA as a more honest baseline than mean-pooled cosine, but it is not uncontested. \citet{davari2022reliability} show that CKA can be made misleading under adversarial column-scaling, and our shared-surface-form alignment introduces a selection effect, since only tokens appearing in both sequences contribute. We mitigate by requiring at least three shared tokens per pair and averaging over middle layers, but the residual selection is real. Linear CKA is rotation- but not translation-invariant; for our purposes this is appropriate, since $\mu$ is exactly the translation we wish to ignore. We test only linear CKA, and kernel or nonlinear variants may behave differently. The theoretical analysis in Eq.~\eqref{eq:prop} assumes isotropic noise, so the closed-form expression is approximate; the qualitative monotonic length-dependence is what we test empirically and is what we observe across domains. Our code experiments cover four models from two families, and the CLIP experiment uses a fixed random-noise image; broader model coverage and natural image--caption pairs would further strengthen the conclusion.

\section{Conclusion}

Mean-pooled cosine similarity, the default metric for comparing neural representations across languages, modalities, and tasks, is not length-invariant under transformer anisotropy. We proved this from first principles, validated it on random vectors with no model involvement, and demonstrated it empirically across code, natural-language, and vision domains. Substituting linear CKA on identical data cuts length-explained variance by $83\%$, flips the sign of $\beta_\text{len}$ in code, and removes the artifact in NLP; the artifact is naturally suppressed in CLIP, exactly where the theory predicts. Yet CKA also shows that genuine cross-language convergence is large---near-saturating in Mistral-7B and 0.91--0.997 across code languages in CodeLlama-13B and Qwen2.5-Coder-7B. The convergence was real; only the asymmetric privileged-pivot framing was metric-induced. We therefore recommend length-invariant metrics as the default for cross-representation comparisons, and a careful re-examination of recent claims of cross-lingual convergence built on mean-pooled cosine.

\bibliographystyle{icml2026}
\bibliography{references}

\clearpage
\appendix
\onecolumn

\begin{center}
{\Large\bfseries Appendix} \\[0.4em]
{\small Supplementary material to ``Mean-Pooled Cosine Similarity is Not Length-Invariant''}
\end{center}
\vspace{1em}

\section{Per-pair raw CKA values}
\label{app:cka_levels}

Table~\ref{tab:cka_levels} reports the raw mean linear CKA at shared token positions, averaged over middle layers, for each model and language pair used in Sec.~\ref{sec:cka_substantive}. High values indicate that representations are similar at matched positions once the $1/\sqrt{n}$ pooling concentration has been removed. Variability across pairs is reported as the standard deviation across problems (code) or sentence pairs (NLP). Values are computed from the same data used for the regression results in Sec.~\ref{sec:results}.

\begin{table}[h]
\centering
\caption{Raw mean CKA at shared token positions, with per-pair standard deviation. Values verified against the source JSON data files; std values reflect across-pair (not bootstrap) variation.}
\label{tab:cka_levels}
\small
\begin{tabular}{@{}llccc@{}}
\toprule
Domain & Pair & Mean CKA & Std & $n$ \\
\midrule
Code     & CodeLlama-7B: Py vs.\ \{Java, JS, Go\}         & 0.650  & 0.189  & 164 \\
Code     & CodeLlama-7B-Python: Py vs.\ \{Java, JS, Go\}  & 0.523  & 0.242  & 164 \\
Code     & CodeLlama-13B: Py vs.\ \{Java, JS, Go\}        & 0.997  & 0.001  & 164 \\
Code     & Qwen2.5-Coder-7B: Py vs.\ \{Java, JS, Go\}     & 0.913  & 0.063  & 164 \\
NLP (FR) & Mistral-7B: EN vs.\ FR                         & 0.9997 & 0.001  & 442 \\
NLP (DE) & Mistral-7B: EN vs.\ DE                         & 0.9997 & $<$0.001 & 428 \\
\bottomrule
\end{tabular}
\end{table}

\end{document}